\documentclass{article}

\usepackage[preprint,nonatbib]{neurips_2025}

\usepackage[utf8]{inputenc} % allow utf-8 input
\usepackage[T1]{fontenc}    % use 8-bit T1 fonts
\usepackage{hyperref}       % hyperlinks
\usepackage{url}            % simple URL typesetting
\usepackage{booktabs}       % professional-quality tables
\usepackage{amsfonts}       % blackboard math symbols
\usepackage{nicefrac}       % compact symbols for 1/2, etc.
\usepackage{microtype}      % microtypography
\usepackage{xcolor}         % colors
\usepackage{amsmath}
\usepackage{mathtools}
\usepackage{amsthm}
\def\R{\mathbb{R}}
\def\E{\mathbb{E}}
\def\ab{\alpha\beta}
\def\cip{\xrightarrow{p}} 
\newtheorem{thm}{Theorem}[section]
\newtheorem{lem}[thm]{Lemma}

\title{Approximating Simple ReLU Networks based on Spectral Decomposition of Fisher Information}

\author{%
  Ka Long Keith Ho \\
  Joint Graduate School of Mathematics for Innovation\\
  Kyushu University\\
  Fukuoka, Japan 819-0395\\
  \texttt{ho.kalongkeith.224@s.kyushu-u.ac.jp} 
  \And
  Yoshinari Takeishi\\
  Faculty of Information Science and Electrical Engineering \\
  Kyushu University\\
  Fukuoka, Japan 819-0395\\
  \texttt{takeishi@inf.kyushu-u.ac.jp} 
  \AND
  Jun’ichi Takeuchi\\
  Faculty of Information Science and Electrical Engineering \\
  Kyushu University\\
  Fukuoka, Japan 819-0395\\
  \texttt{tak@inf.kyushu-u.ac.jp} \\
}

\begin{document}

\maketitle

\begin{abstract}
Properties of Fisher information matrices of 2-layer neural ReLU networks with random hidden weights are studied. For these networks, it is known that the eigenvalue distribution highly concentrates on several eigenspaces approximately. In particular, the eigenvalues for the first three eigenspaces account for 97.7\% of the trace of the Fisher information matrix, independently of the number of parameters. In this paper, we identify the function spaces which correspond to those major eigenspaces. This function space consists of the spherical harmonic functions whose orders are not greater than 2. This result relates to the Mercer decomposition of the neural tangent kernels.
\end{abstract}

\section{Introduction}

We study the function space represented by a two-layer
(one hidden layer, bias-free)
ReLU neural network with randomly generated
fixed hidden weights, based on the properties of 
the Fisher information matrix shown by
Takeishi et al.\
\cite{Takeishi,TakeishiNN,Takeishi24}.

In our analysis, we assume that the input dimension is $d$,
the number of the hidden units is $m$, and the output is one-dimensional.
Let $f_v(x)$ denote the function the neural network represents,
where $v\in \mathbb{R}^{m \times 1}$ is a weight vector for the output layer.
We assume that the input \(x\) follows the \(d\)-dimensional standard
Gaussian distribution and that the output variable \(y\) is generated by
\begin{align}\label{ReluModel}
    y=f_v(x)+\epsilon
    =X(x)^\top v+\epsilon,
\end{align}
where
\(X=X(x)=\sigma(x^\top W)\in\mathbb{R}^{m\times 1}\) denotes the output
of the hidden layer with the ReLU activation function \(\sigma(\cdot)\)
applied component-wise, and \(\epsilon\) is a standard Gaussian noise.
$W \in \mathbb{R}^{d \times m}$ is the weight matrix
from the input layer to the hidden layer
whose $(i,j)$-entry is the weight from the $i$th input node
to the $j$th hidden node.
% Our analysis is based on 
% the Fisher information $J$ of $v$ under this assumption.

In this setting, Takeishi et al.\ derived
an approximate spectral decomposition of $J$,
showing that
its eigenvalue distribution 
highly concentrates and the sum of the top $O(d^2)$ eigenvalues
accounts for 97.7\% of its trace at least,
no matter how $m$ is large \cite{Takeishi,TakeishiNN,Takeishi24}.

In this paper,
we identify the function space which
is induced by the eigenspaces with the eigenvalues that account for
97.7\% of the trace of $J$. 
It means that we derive an 
$O(d^2)$-dimensional linear model,
which is equivalent to the model $f_v$, 
ignoring 2.3\% portion of the target neural network model.

The approximate spectral decomposition 
of $J$ is summarized as follows.
Note that the terms {\it eigenvector} and {\it eigenvalue} in
the following means approximate ones.
The first eigenvalue nearly equals
$(2d+1)/4\pi$, the second is $1/4$ with the multiplicity $d$,
and the third is $1/2\pi(d+2)$ with the multiplicity $d(d+1)/2-1$.
In addition, the corresponding eigenvectors are identified.
The eigenvector with the first eigenvalue, denoted $v^{(0)}$ 
is close to a uniform vector.
To be more specific, it is given as
\[
v^{(0)} = 
\biggl(\frac{\Vert W_{*1}\Vert}{\sqrt{d}},...,
\frac{\Vert W_{*m}\Vert}{\sqrt{d}}\biggr)^\top.
\]
Here,  
$W_{*j}$ denotes the $j$th column vector of $W$.
The eigenvectors with the eigenvalue $1/4$
are $v^{(l)} = W_{l*}^\top$ for each $l \in \{1,\ldots,d \}$, where $W_{l*}$ denotes the $l$th row vector of $W$.

To analyze the function space,
we assume that each entry of $W$ is independently drawn according to 
the Gaussian distribution with the mean $0$ and the variance
$1/m$. It means that $v^{(0)}$ and $v^{(l)}$
are approximately normalized with high probability.

The eigenspace with the eigenvalue $1/2\pi(d+2)$
 is spanned by the vectors
$v^{(\alpha,\beta)}$ ($1 \le \alpha < \beta \le d$)
and $v^{(\gamma)}$ ($1 \le \gamma \le d$)
defined as
% A group consists of 
% the vectors $v^{(\alpha,\beta)}$ 
% for $1\le \alpha < \beta \le d$, which are 
% defined as
\begin{align*}
v^{(\alpha,\beta)} 
& \propto
\biggl(
\frac{W_{\alpha 1}W_{\beta 1}}{\|W_{*1}\|},
\ldots,
\frac{W_{\alpha m}W_{\beta m}}{\|W_{*m}\|}
\biggr)^\top, \\
\tilde{v}^{(\gamma)}
& \propto
v^{(\gamma,\gamma)} -\frac{\sum_{\gamma'=1}^d v^{(\gamma',\gamma')}}{d}.
\end{align*}
Note that $v^{(\alpha,\beta)}$ for $\alpha = \beta$ is used to
define $\tilde{v}^{(\gamma)}$.
% for $1\le \alpha \le \beta \le d$.
% The $j$th entry of $v^{(\alpha \beta)}$
% is proportional to $W_{\alpha j}W_{\beta j}/\|W_{*j}\|$
% for $\alpha \neq \beta$. 
%, where $W_{*j}$ denotes the $j$th column vector of $W$.
%Another group consists of
%for $1 \le \gamma \le d$.
The manipulation obtaining $\tilde{v}^{(\gamma)}$
from $v^{(\gamma,\gamma)}$ is needed to remove the component
parallel to $v^{(0)}$.
The dimension of the space spanned by $\tilde{v}^{(\gamma)}$ 
is $d-1$. %even if we add $\tilde{v}^{(\gamma)}$. 
As a result, the dimension of this eigenspace is $d(d+1)/2-1$.
Finally, we define $v^{(\gamma)}$ ($1 \le \gamma \le d-1$)
by orthogonalization of $\tilde{v}^{(\gamma)}$.

Note that the trace of the Fisher information matrix nearly equals
$d/2$ with high probability with respect to the distribution of $W$.
Using this fact, 
it was shown that
the sum of all the eigenvalues corresponding to the three 
eigenspaces is more than 97.7\% of $d/2$.

This fact suggests that the 3 eigenspaces 
are sufficient to approximate the function space
represented by $f_v$.
% In fact, if the range of $v$ is restricted as $\max_i | v_i | \le C$ 
% with a certain $C > 0$,
% the Kullback-Leibler (KL) divergence loss incurred 
% by ignoring $m - O(d^2)$ parameters
% is bounded by $0.023 \cdot D$, where $D$ denotes the maximum value
% of the KL divergence from the true function to the estimated function.
This observation is our motivation to 
analyze the function space corresponding to the three eigenspaces
of the parameter $v$.

Our main results in this paper are summarized as follows 
(the convergence is in probability).
%For the eigenvectors $v^{(0)}$, $v^{(l)}$, and $v^{\alpha,\beta)}$,
%we show the following convergence results.
\begin{align*}
\lim_{m \rightarrow \infty}f_{v^{(0)}} &=
F_0(x)
\propto \Vert x \Vert, \\
\lim_{m \rightarrow \infty}
f_{v^{(l)}}
&=
F_l(x) \propto \frac{x_l}{2},\\
\lim_{m \rightarrow \infty}
f_{v^{(\alpha,\beta)}}
&=
F_{\alpha\beta}(x) \propto \frac{x_\alpha x_\beta}{\| x\|},
\end{align*}
% ($l_2$-norm of the vector $x$)
where 
$1 \le l \le d$ and $1 \le \alpha \le \beta \le d$.
%where $\eta(d)=O(d^{-3/2})$ for typical $x$.
Note that 
%ignoring the $\eta(d)$,
we can obtain an orthonormal basis of the function space
corresponding to the 
third eigenspace of $v$
by the same algebraic manipulation to
obtain $v^{(\gamma)}$ from $v^{(\gamma,\gamma)}$.
Hence, we introduce $F_\gamma$
by the same manner. 
We employ 
$F_{\alpha \beta}$ ($1 \le \alpha < \beta \le d$)
and 
$F_\gamma$ ($1 \le \gamma \le d-1$)
as the orthogonal basis of the function space
corresponding to the third eigenspace of $v$.

To investigate the implication of our results,
we consider the metric of the function space corresponding to the
Fisher information.
For the model $y=f_v(x)+\epsilon$, 
the Fisher information of $v$ is induced by 
the inner product of the function space
over $\mathbb{R}^d$ (the standard $l_2$ norm w.r.t.\ $N(0,I_d)$)
which is defined as
$
\langle f, g \rangle = \E_{x \sim N(0,I_d)}\left[ f(x) g(x) \right].
$
We consider the linear space $\mathcal{H}$ which 
consists of all the functions $f$ with $\langle f, f \rangle < \infty$.
In our setting, the above inner product corresponding the Fisher metric
for the regression model $\{ y=f(x)+\epsilon : f \in \mathcal{H} \}$.

Here, the set  $\{ X^\top v : v \in \mathbb{R}^m\}$ forms a linear subspace of $\mathcal{H}$.
Then, we have $J_{ij} = \langle X^\top e_i, X^\top e_j \rangle$,
where $e_i$ denotes the standard unit vector 
of $\mathbb{R}^m$
whose $i$th component is $1$.
We also have $\langle f_u, f_v \rangle$
$=\langle X^\top u, X^\top v \rangle$
$=u^\top Jv$.

% Building on this perspective, we shift our focus from the parameter space to 
% the function space associated with the ReLU network $f_v(x) = X^\top v$. 

% In this setting, the basis function aligned with a particular parameter direction $v^{(\cdot)}$ — an approximate eigenvector of the FIM as identified by Takeishi et al. \cite{Takeishi} — is given by $f_{v^{(\cdot)}}(x)$. 
% This implies that, in function space, gradient descent initially advances in the direction of this basis function. In this paper, we analyze  the limiting function to which the basis function  $f_{v^{(\cdot)}}(x)$ converges as the number of hidden units $m$ becomes large.

Since we define $F_0$, $F_l$, $F_{\gamma}$, 
and
$F_{\alpha \beta}$ based on the approximate unit eigenvectors of 
the three eigenvalues of $J$, 
their norms equal the square root of the corresponding eigenvalues, respectively. The orthogonality between them also approximately holds.

Let us employ a new parameter $u \in \mathbb{R}^{m \times 1}$ 
which diagonalizes $J$ by an orthogonal matrix $U$ as $v = U u$. 
Then, when $m$ is large, we can approximate $f_v(x)$ with high probability as
\begin{align}\label{theapproximation}
f_{Uu}(x) & =  u_1 F_0(x) + \sum_{l=1}^{d} u_{l+1} F_l(x) \\ \nonumber
+ & \sum_{\gamma=1}^{d-1} u_{d+1+\gamma} F_{\gamma}(x)
  + \sum_{\alpha < \beta} u_{k(\alpha,\beta)} F_{\alpha \beta}(x)
  + r_u(x),
\end{align}
where $k(\alpha,\beta)$ is a bijection from the range of $(\alpha,\beta)$
onto the set 
$\{2d+1,2d+2,\ldots, d+d(d+1)/2 \}$
and a residual term $r_v(x)$ is small.
% More specifically,
% \[
% \max_{u : \| u \|_\infty =1}
% \| r_u(x)\|_F^2 \le 0.023 \cdot
% \max_{u : \| u\|_\infty =1}\| f_{Uu}(x)\|_F^2
% \]
% holds, 
% %when $\Vert v \Vert \le 1$ is satisfied.
% because of the concentration of the eigenvalue distribution of $J$,
% where $\| u\|_\infty$ denotes the max norm of $u$
% and $\| f\|_F^2 = \langle f, f \rangle$.
% The restriction  $\Vert v \Vert \le 1$ is commonly assumed in 
% many recent theoretical researches 
% on neural networks and is a realistic assumption in practical scenes, too.

% Note that $u$ is restricted as $\max_i |u_i| \le 1$, under the restriction
% $\| v \| \le 1$. 
% Hence, under this restriction, we conclude that 
% the squared norm of
% % the significance of 
% each of $F_0$, $F_l$, and $F_{\alpha \beta}$ ($F_{\gamma}$) 
% approximately equals
% $(2d+1)/4\pi$, $1/4$, and $1/2\pi(d+2)$, respectively.

% Recall that $F_0(x) \propto \Vert x \Vert$, $F_l(x) = x_l /2$, and that $\tilde{F}_\gamma$ and $F_{\alpha \beta}$ represent the product of $x_i$ and $x_j$. This observation suggests that these features of $x$ can be effectively learned by our model $f_v(x) = X^\top v$.

Moreover, we can get an insight into gradient descent using \eqref{theapproximation}.
First note that the gradient descent
with respect to the new parameter $u$
is equivalent to that with respect to the
original parameter $v$, since
$u$ is obtained by an isometric mapping of $v$.
Recall that the convergence rate of the the gradient descent is determined by eigenvalues of the Hessian matrix of training error. 
%(See Exercise 5.25 of \cite{prml} for example.)
In fact, the training for the direction of the eigenvectors of large eigenvalues is fast. Note that the Hessian matrix is approximately equal to the FIM when 
the estimate is near the optimal point and 
the number of data is large enough
compared to the number of parameters. This condition may be relaxed for our case, where the eigenvalue distribution is highly concentrates. 
%Further, our model is a linear regression model, so the Hessian and FIM are constant for $v$. 
Hence, we can assume that the convergence rate may be determined solely by the eigenvalues of the FIM. Then, we can claim by \eqref{theapproximation} that the training about the feature $\Vert x \Vert$ is fastest, the training about the features $x_l$ is next, and the training about the features related to $x_i x_j$ follows. 
%Note that we can estimate the concrete convergence rates.

% The above insight about the learning dynamics
% gives a concrete picture of the 
% training trajectory in
% the function space.
% It is relevant to
% the theory of neural tangent kernel (NTK)
% \cite{Jacot}.
% It gives a very nice insight for multi-layer
% neural networks, which are much more
% general than our target, but
% does not give a concrete picture.
% We should say that our result
% is obtained by concentrating
% the simplest cases.

It is interesting that
the functions $\| x \|$, $x_l /2$, and $x_\alpha x_\beta/\| x \|$
relate to the spherical harmonic functions over $S_{d-1}$ 
(the surface of the $d$-dimensional unit sphere).
The spherical harmonic functions of degree $k$
over $S_{d-1}$
are given by restricting the 
the homogeneous polynomials of $z \in \R^d$ of degree $k$
which satisfy the Laplace's equation to $S_{d-1}$.
It is known that the spherical harmonic functions
gives an orthonormal basis of the set of (appropriately defined)
functions over $S_{d-1}$
with respect to the inner product 
$\langle f, g \rangle = \int_{S_{d-1}} f(x)g(x) dS$,
where $dS$ is the uniform probability measure over $S_{d-1}$.

Note that,
for a spherical harmonic function $\psi(z)$,
the function $\|x \|\psi(x /\| x\|)$ is defined over $\R^d$.
The functions $\| x \|$, $x_l /2$, and $x_\alpha x_\beta/\| x \|$
can be obtained by this manipulation.
For example, $\| x \|$ is obtained by $\psi(z)=1$
and $x_\alpha x_\beta/\| x\|$ by $\psi(z)=z_\alpha z_\beta$.
In fact, our result closely related to the Mercer decomposition 
of neural tangent kernel (NTK) for two-layer networks,
which was shown and  
discussed by Bietti and Ma \cite{Bietti19},
and discussed in \cite{Bartlett}.
Note that their analysis is about more general model
than ours, that is, $W$ is not assumed to be constant,
but they decomposed the NTK to those for $W$ and $v$.
Let $k(x_1,x_2)$ denote the neural tangent kernel for $v$.
%which is obtained with the limit $m \rightarrow \infty$.
Note that $k(x_1,x_2)$ defines a linear mapping over the function space
on $\R^d$.
The Mercer decomposition is the spectral decomposition
of $k(x_1,x_2)$ (as a linear mapping)
using the spherical harmonic functions as the orthonormal basis.
% in which the eigenvalue is determined by the order of 
% a spherical harmonic function.
The Mercer decomposition is 
parallel to the approximate spectral decomposition of $J$.
In fact, the linear mapping induced by $J$ is nearly isomorphic to
the linear mapping by $k(x_1,x_2)$.
(It is `nearly' because $J$ is defined for finite $m$.)
Hence, our result can be interpreted as it approximately 
identifies the eigenvalues
of the first three degrees of the Mercer decomposition.
The works \cite{Bietti19,Bartlett} describe the
tail behavior of
eigenvalue distribution of $k(x_1,x_2)$, that is,
they state $\lambda_k = O(k^{-d})$, where the 
order is defined as $k$ goes to infinity,
but they do not give the concrete values. 
%for the small $k$.

% It seems weak that our result is about only the first three orders,
% but we do not think so, because the sum of these eigenvalues
% account for 97.7\% of the trace of $J$.
% In addition, it may be interesting that the our result is
% obtained based on the approximate spectral decomposition,
% whose proof does not depend on the properties of
% the spherical harmonic functions.
% Note that, if we employ the next term of the 
% approximate spectral decomposition of $J$,
% then the sum of eigenvalues account for 99\% of the trace.

% The FIM induces a function space inner product,
% \[
% \langle f_v, f_u \rangle = v^\top J u = \E\left[(X^\top v)(X^\top u)\right],
% \] 
% which approximates orthogonality for basis functions aligned with $J$’s eigenvectors, with norms proportional to the square roots of the corresponding eigenvalues. This structure bridges parameter and function spaces, offering a novel lens on neural network expressivity. In the infinite-width limit, $J$’s kernel converges to the Neural Tangent Kernel (NTK) \cite{Jacot}, coinciding with the arc-cosine kernel for our shallow ReLU network. This suggests that $f_{v^{(k)}} (x)$ may act as approximate NTK eigenvectors, a direction we explore theoretically and defer to future empirical analysis.

%{\bf Related works}

In our two-layer neural network setting, the first-layer weights are randomly sampled and fixed, and only the second-layer weights are trained. This framework was first introduced in \cite{RahimiRecht} and is widely studied under the name of random feature regression, and its generalization performance has been reviewed, for example, in \cite{Bartlett} and \cite{MeiMontanari}. Contrary to the generalization error analysis performed in these works, our work examines the function space dynamics of a ReLU network, deriving the asymptotic functional forms of basis functions aligned with the FIM's approximate eigenvectors. Both papers employ rigorous probabilistic tools to study overparametrized models, with our analysis extending to the interplay between parameter and function spaces via the FIM-induced inner product, complementing their insights into generalization performance. This shared emphasis on random feature models and asymptotic analysis underscores a common goal of understanding the expressive and generalization capabilities of wide neural networks.

In the field of neural networks,
a similar concept to the above research
line, named 
neural networks with random weights (NNRW)
 has been studied. (See \cite{nnrw}.)

\section{Preliminaries}

We introduce important notations and 
explain necessary theoretical background. 

\subsection{Notations}
For a vector $v$ and matrix $A$, we use $\Vert v \Vert$ to denote the $l_2$ norm and $\Vert A \Vert$ for its spectral norm. We also denote $v^\top$ and $A^\top$ for their transposes and for $i,j \in \mathbb{Z^+}$, $v_{-i}$ and $v_{-ij}$ as the vector $v$ with the $i$th, or $i$th and $j$th components removed. We will use $\phi$ and $\Phi$ as the density and cumulative distribution functions of the standard normal distribution, $B(\cdot,\cdot)$ as the beta function, $\cip$ for convergence in probability as $m \to \infty$, and LHS and RHS to abbreviate for the left and right-hand sides of equations.

\subsection{Approximate Eigenvectors of Fisher information matrix}\label{sec_approx_eigen}
In this subsection, we review the approximate eigenvectors of Fisher information matrix (FIM) derived by \cite{Takeishi,TakeishiNN,Takeishi24}.
We study the estimation of the parameter $v$ in the model \eqref{ReluModel} and its FIM, defined by
\[
J = \mathbb{E}_{x \sim N(0, I_d)}[X(x) X^{\top}(x)].
\]
When the entries of the weight matrix $W$ are generated as i.i.d.\ $N(0,1/m)$ random variables, the eigenvalue distribution of $J$ exhibits clustering. 
In \cite{Takeishi,TakeishiNN,Takeishi24}, the eigenvalues and eigenvectors of the first three eigenvalue clusters were analyzed.

%Let $x \in \R^d$ be the input, $W \in \R^{d\times m }$ be a fixed matrix whose entries are generated as i.i.d. $N(0,1/m)$ random variables, and let $X = \sigma(x^\top W) \in \R^m$ be the ReLU activated covariate, where the ReLU function $\sigma(\cdot)$ is applied component-wise. Then, let $v$ be an $m$-dimensional vector, we will determine the asymptotic limit of $X^\top v$ in the cases where $v$ belongs to one of the three clusters of eigenvectors identified by \cite{Takeishi24}. 

% To summarize their results, the first group consists of a single eigenvector 
% \[
% v^{(0)} \coloneqq \left(\Vert W_{*1}\Vert/\sqrt{d},...,\Vert W_{*m}\Vert/\sqrt{d}\right)^\top.
% \]
% %with $W^{(i)} \in \R^d$ for $i = 1,...,m$, where $W^{(i)}$ 
% %denotes the $i^{th}$ column of the $W$. 
% Its approximate eigenvalue is $(2d+1)/4\pi$. The second group contains $d$ vectors of them form \[v^{(l)} \coloneqq W_{l*}^\top, \quad \text{for} \quad l = 1,...,d,\]
% where $W_{l*}$ stands for the $l$th row vector of $W$, 
% with an approximate eigenvalue of $1/4$. 
Here, we give an orthonormal basis of the third eigenspace of $J$.
Let the $i$th element of $v^{(\alpha,\beta)} \in \R^{m \times 1}$ be
\[
v^{(\alpha,\beta)}_i 
= 
\sqrt{d + 2}
\frac{W_{\alpha i} W_{\beta i}}
{\Vert W_{*i} \Vert} \quad \text{ for } \quad 1\leq\alpha\leq\beta\leq d.
\]
The third group consists of $d-1$ vectors of the form 
\begin{align}\label{def_vgam}
v^{(\gamma)}=\tilde{v}^{(\gamma)}-\frac{1}{\sqrt{d}+1}\tilde{v}^{(d)} \quad \text{for} \quad\gamma = 1,...,d-1,
\end{align}
where
\begin{equation}\label{vgam_expression}
    \tilde{v}^{(\gamma)} \coloneqq \frac{1}{\sqrt{2}}\left(v^{(\gamma,\gamma)} - \sqrt{\frac{d+2}{d}} \cdot 
v^{(0)}\right)\quad \text{for } \gamma = 1,...,d,
\end{equation}
and $(d^2 - d)/2$ vectors \[v^{(\alpha,\beta)} \quad \text{for} \quad 1\leq \alpha < \beta \leq d,\] whose approximate eigenvalues are $1/(2\pi (d+2))$. %As shown in \cite{Takeishi}, 
Then, the FIM can be written as 
\begin{align}\label{FIMdecomp}
  J &=  \frac{2d+1}{4\pi} v^{(0)} v^{(0)\top}
   + \frac{1}{4} \sum_{l = 1}^d v^{(l)} v^{(l)\top}  \\ \nonumber
   & + \frac{1}{2\pi (d+2)}\left(
       \sum_{\gamma =1}^{d-1} v^{(\gamma)} v^{(\gamma)\top}
       + \sum_{\alpha < \beta} v^{(\alpha,\beta)} v^{(\alpha,\beta)\top}
     \right)  \\ \nonumber
     & + R,
\end{align}
where the residual $R$ is semi-positive definite
and its trace is less than 
0.023$\cdot$trace($J$) with high probability.

\subsection{Fisher Metric and Function Space Inner Product}\label{sec_fisher_metric}

To provide deeper insight into the function space perspective of the simple ReLU network $f_v(x) = X^{\top} v$, we discuss the role of the Fisher information matrix (FIM) and its associated metric in connecting the parameter and function spaces. The FIM, defined as $J = \mathbb{E}_{x \sim N(0, I_d)}[X(x) X^{\top}(x)]$, induces an inner product in the parameter space via $\langle u, v \rangle_J = u^{\top} J v$. This inner product reflects the Fisher metric, which is intrinsically defined on the space of probability distributions and is independent of the parameterization.

For a probability distribution $p_\theta(x)$, the Fisher metric quantifies the squared norm of a small change in the log-likelihood, $d \log p_\theta(x)$, as $\mathbb{E}[(d \log p_\theta(x))^2]$. In our setting, where $x \sim N(0, I_d)$, this leads to an inner product in the function space defined by:
\[
\langle X, X' \rangle = \mathbb{E}_{x \sim N(0, I_d)}[X(x) X'(x)].
\]
Initially, one might assume this corresponds to the standard inner product in a Hilbert space defined by simple integration. However, due to the normal distribution assumption on $x$, this inner product is weighted by the Gaussian density, distinguishing it from the standard $L^2$-inner product.

This function space inner product has significant implications for the basis functions $f_v(x) = X^{\top} v$. Specifically, the inner product between two basis functions is:
\[
\langle f_v, f_u \rangle = \mathbb{E}_{x \sim N(0, I_d)}[(X^{\top} v)(X^{\top} u)] = v^{\top} J u,
\]
demonstrating that the function space inner product aligns with the FIM-induced inner product in the parameter space. For the approximate eigenvectors $v_i$ of $J$, as identified by \cite{Takeishi} and defined in Section 2.2, we have:
\[
\langle f_{v_i}, f_{v_j} \rangle = v_i^{\top} J v_j \approx \lambda_i \delta_{ij},
\]
where $\lambda_i$ are the approximate eigenvalues, and the approximation holds due to the finite dimensionality $m$. This implies that the basis functions $f_{v_i}$ are approximately orthogonal in the function space, with norms approximately equal to $\sqrt{\lambda_i}$. These properties are meaningful for understanding the learning dynamics in the function space, as gradient descent initially progresses along directions corresponding to the leading basis functions (see Section 6 of~\cite{TakeishiNN}).

A deeper connection exists between the FIM and the Neural Tangent Kernel (NTK), as noted in Section 1. In the infinite-width limit, the kernel associated with $J$ converges to the NTK, suggesting that the basis functions $f_{v_i}$ may serve as approximate eigenvectors in the function space. However, this analysis is complex and deferred to future work.

\section{Main Results}\label{sec_main}
We determine the asymptotic limit of $f_v(x) = X^{\top} v$ as $m \to \infty$, when $v$ lies in one of the three clusters of eigenvectors described in Section~\ref{sec_approx_eigen}, where the entries of $W$ are assumed to be i.i.d.\ $N(0,1/m)$ random variables.
We present four theorems corresponding to these three groups of approximate eigenvectors (Theorems \ref{thm3mod} and \ref{thm4} correspond to the third group). The proofs will be deferred to the appendices.
\begin{thm}\label{thm1}
Let $d > 2$. For each $x \in \R^d$, the first approximate eigenvector $v^{(0)}$ satisfies
\begin{equation}\label{G1Result}
    X^\top v^{(0)} \cip \frac{\sqrt{d}}{2\pi}B\Bigl(\frac{d}{2},\frac{1}{2}\Bigr)\Vert x \Vert \coloneqq F_0(x).
\end{equation}
\end{thm}

\begin{thm}\label{thm2}
Let $d > 2$. For each $x \in \R^d$, the second group of approximate eigenvectors $\{v^{(l)}; l = 1,...,d\}$ satisfies
\begin{equation}\label{G2Result}
    X^\top v^{(l)} \cip \frac{x_l}{2} \coloneqq F_l(x).
\end{equation}
\end{thm}

Since the analysis of $v^{(\gamma)}$ is complicated, we first analyze $\tilde{v}^{(\gamma)}$ in the following lemma.

\begin{lem}\label{thm3}
Let $d > 2$. For each $x \neq 0 \in \R^d$, the vectors $\{\tilde{v}^{(\gamma)}; \gamma = 1,...,d \}$ satisfy
%\begin{equation}\label{G3Result1}
%    X^\top \tilde{v}^{(\gamma)} \cip 
%    \begin{cases}
%        \frac{(d-1)\sqrt{d+2}}{2\pi(d+1)\sqrt{2}}B\bigl(\frac{d}{2},\frac{1}{2}\bigr) \Vert x \Vert & \text{if} \quad x_{-\gamma} = 0 \\
%        -\frac{\sqrt{d+2}}{2\pi(d+1)\sqrt{2}}B(\frac{d}{2},\frac{1}{2}) \Vert x \Vert & \text{if} \quad x_\gamma = 0
%    \end{cases} \quad \text{as} \quad m \to \infty.
%\end{equation}
%Otherwise, 
\begin{align}\label{3a_general}
    X^\top \tilde{v}^{(\gamma)} \cip &\frac{d\sqrt{d+2}}{2\pi(d+1)\sqrt{2}}
    B\Bigl(\frac{d}{2},\frac{1}{2}\Bigr)  \left( \frac{x_\gamma^2}{\Vert x \Vert}- \frac{\Vert x \Vert}{d}\right).
\end{align}
\end{lem}

{\bf Remark:} 
The right hand side of \eqref{3a_general} is proportional to the difference between $x_\gamma^2/\Vert x \Vert$ and their mean for $\gamma$, which equals $\Vert x \Vert/d$. Note that $\Vert x\Vert$ is proportional to \eqref{G1Result}.

\begin{thm}\label{thm3mod}
Let $d > 2$. For each $x \neq 0 \in \R^d$, the first $d-1$ approximate eigenvectors $\{v^{(\gamma)}; \gamma = 1,...,d-1 \}$ in the third group satisfy
\begin{align*}
&X^\top v^{(\gamma)} \cip \; \\&
\frac{d\sqrt{d+2}}{2\pi(d+1)\sqrt{2}} B\Bigl(\frac{d}{2},\frac{1}{2}\Bigr) \left( 
    \frac{x_\gamma^2}{\Vert x \Vert} - \frac{1}{\sqrt{d}+1} \frac{x_d^2}{\Vert x \Vert} - \frac{\Vert x \Vert}{d+\sqrt{d}} 
\right) \\
&\coloneqq F_{\gamma}(x).
\end{align*}
\end{thm}

\begin{thm}\label{thm4}
Let $d > 2$. For each $x \neq 0 \in \R^d$, the approximate eigenvectors $\{v^{(\alpha,\beta)}; 1 \leq \alpha < \beta \leq d\}$ satisfy
%\begin{equation}
%    X^\top v^{(\alpha,\beta)} \cip \begin{cases}
%        0 & \text{if } x_\alpha = 0 \text{ or } x_\beta = 0\\
%        \frac{d\sqrt{d+2}}{2\pi(d+1)}B(\frac{d}{2},\frac{1}{2})\frac{x_\alpha x_\beta}{\Vert x \Vert} & \text{if } x_{-\ab} = 0.
%    \end{cases}
%\end{equation}
%Otherwise, 
\begin{align}\label{3b_general}
        X^\top v^{(\alpha,\beta)} &\cip \frac{d\sqrt{d+2}}{2\pi(d+1)}B\Bigl(\frac{d}{2},\frac{1}{2}\Bigr)\frac{x_\alpha x_\beta}{\Vert x\Vert} \coloneqq F_{\ab}(x).
\end{align}

\end{thm}

From these theorems, we see that when $m$ is sufficiently large, the basis functions primarily learned by gradient descent are proportional to
$\Vert x\Vert$, $x_l$ ($l=1,\ldots, d$), $(x_\gamma^2-x_d^2/(\sqrt{d}+1))/\Vert x \Vert- \Vert x \Vert/(d+\sqrt{d})$ ($\gamma =1,\ldots, d-1$), and $x_\alpha x_\beta /\Vert x \Vert$ ($1\leq \alpha < \beta \leq d$). 

In Section \ref{sec_fisher_metric}, we also discussed the approximate orthogonality of these basis functions when $x$  is generated by a standard multivariate normal distribution. Specifically, for approximate eigenvectors $v_i$ and $v_j$, \[
\langle f_{v_i}, f_{v_j} \rangle = v_i^{\top} J v_j \approx \lambda_i \delta_{ij}.
\] 
For $i \neq j$, $\langle f_{v_i}, f_{v_j} \rangle = 0$ can be shown using the explicit forms from the theorems and the pairwise independence of $\Vert x \Vert$, $x_l/\Vert x \Vert$, and $x_k/\Vert x \Vert $ with $l \neq k$.

For sufficiently large $d$, when $i = j$, approximating $B(d/2,1/2) \approx \sqrt{\pi(2d+1)}/d$ allows us to recover the respective approximate eigenvalues for each group of eigenvectors in \eqref{FIMdecomp}:
\[\E_{x \sim N(0,I_d)}\left[ (X^\top v^{(0)})^2\right] \approx \left( \frac{\sqrt{d}}{2\pi}\frac{\sqrt{\pi(2d+1)}}{d}\right)^2\E[\Vert x\Vert^2] = \frac{2d+1}{4\pi},\]
\[\E_{x \sim N(0,I_d)}\left[ (X^\top v^{(l)})^2\right] = \frac{1}{4}\E[x_l^2] = \frac{1}{4} \quad \text{for}\quad l = 1,...,d,\]
\begin{align*}
    \E_{x \sim N(0,I_d)}\left[ (X^\top {v}^{(\gamma)})^2\right] &\approx \left( \frac{d\sqrt{d+2}}{2\pi(d+1)\sqrt{2}}\frac{\sqrt{\pi(2d+1)}}{d}\right)^2\\
    &\quad \times \E\left[\Vert x\Vert^2 \left(\frac{x_\gamma^2}{\Vert x\Vert^2} -\frac{x_d^2}{\Vert x\Vert^2(\sqrt{d}+1)}- \frac{1}{d + \sqrt{d}}\right)^2\right]\\
    &=\frac{(d+2)(2d+1)}{8\pi(d+1)^2}\E\bigg[ \frac{x_\gamma^4}{\Vert x\Vert^2}  + \frac{x_d^4}{\Vert x\Vert^2(\sqrt{d}+1)^2}\\
    &\quad+ \frac{\Vert x\Vert^2}{(d + \sqrt{d})^2} - \frac{2x_\gamma^2x_d^2}{\Vert x\Vert^2(\sqrt{d}+1)} - \frac{2 x_\gamma^2}{d + \sqrt{d}} + \frac{2 x_d^2}{\sqrt{d}(\sqrt{d}+1)^2}\bigg]\\
    &= \frac{(d+2)(2d+1)}{8\pi(d+1)^2}\E\bigg[ \frac{3}{d+2} + \frac{3}{d+2(\sqrt{d}+1)^2}\\
    &\quad+ \frac{1}{(\sqrt{d}+1)^2} - \frac{2}{(d+2)(\sqrt{d}+1)} - \frac{2}{d + \sqrt{d}} + \frac{2}{\sqrt{d}(\sqrt{d}+1)^2}\bigg]\\
    &\approx \frac{1}{2\pi (d+2)} \quad \text{for} \quad \gamma = 1,...,d-1,
\end{align*}
\begin{align*}
\E_{x \sim N(0,I_d)}\left[ (X^\top v^{(\alpha,\beta)})^2\right] &\approx \left( \frac{d\sqrt{d+2}}{2(d+1)\pi}\frac{\sqrt{\pi(2d+1)}}{d}\right)^2\E\left[\frac{x_\alpha^2x_\beta^2}{\Vert x\Vert^2}\right]\\
&= \frac{(d+2)(2d+1)}{4(d+1)^2\pi}\frac{1}{d+2}\\
&\approx \frac{1}{2\pi (d+2)} \quad \text{for }\quad 1\leq\alpha < \beta \leq d.
\end{align*}

\section{Simulation}\label{sec_sim}

We examine how accurate our main results are by numerical simulation.

\subsection{Setup}\label{sec_setup}
We validate the theoretical results shown in Section \ref{sec_main}. Following the model \eqref{ReluModel}, we set $d = 10, 50, 100$, and generate $N = 100$ independent copies of $x \sim N(0,I_d)$, which reflect realistic training inputs. We also fix the dimensions of the middle layer to be $m = 500, 10000, 100000$, and generate the weight matrix $W$ element-wise as instances of $N(0,1/m)$ i.i.d. random variables. For approximate eigenvectors $v$, we show that the realizations $X^\top v$ are consistent with the asymptotic limits identified. 

The simulations were run on a desktop with a 12th Gen Intel(R) Core(TM) i7-12700KF CPU (12 cores, 20 logical processors, 3.60 GHz) and 32 GB RAM (3200 MHz), using RStudio Desktop on a Windows system with 1 TB SSD storage.

\subsection{Simulation results}\label{sec_sim_res}
We compute the mean absolute error (MAE $= N^{-1}\sum_{i=1}^N |F(x_i) - X(x_i)^\top v|$) for approximate eigenvectors $v$ and their corresponding limiting functions $F$ identified in Section \ref{sec_main}. Note that for the approximate eigenvectors of Group 3, we use $F_{\gamma\gamma}$ and $F_{\ab}$ in \eqref{3a_general} and \eqref{3b_general} without the remainders $h_\gamma$, $h_d$, and $h_{\ab}$.

The results are shown in Table \ref{tab1} and are consistent with the theorems. Unsurprisingly, the accuracy scales with the number of hidden neurons $m$, as reflected by the decreased MAE. It can also be seen that performance drops with the input dimension $d$, which is likely caused by the increased variance of $\Vert x \Vert^2 \sim \chi^2(d)$. We also visualize this fit in Figure \ref{fig1}, where we plot the theoretical values against the values of $X^\top v$ for each group over $N = 100$ instances of $x \sim N(0,I_d)$. 
The theoretical values are shown as red lines, while the values of $X^\top v$ are plotted as blue dots.

\begin{table}[h]
  \centering
  \caption{For approximate eigenvectors $v$ described in Section \ref{sec_approx_eigen}, this table shows the mean absolute error (MAE) between realizations of $X^\top v$ and the asymptotic limits identified in theorems \ref{thm1} to \ref{thm4}.}
  \begin{tabular}{|c|c|c|c|c|c|c|}
    \hline
    $d$ & $m$ & G1 ($v^{(0)}$) & G2 ($v^{(l)}$) & G3 ($v^{(\gamma)}$) & G3 ($v^{(\alpha,\beta)}$)\\
    \hline
    $50$  &   500  & 0.1344 & 0.1782 & 0.1739 & 0.1748\\
    $100$ &   500  & 0.2033 & 0.3874 & 0.3471 & 0.2486\\
    $10$ & 10000 & 0.0159 & 0.0189 & 0.0140 & 0.0159\\
    $50$ & 10000 & 0.0276 & 0.0506 & 0.0368 & 0.0679\\
    $100$& 10000 & 0.0481 & 0.0986 & 0.0452 & 0.0613\\
    $10$ &100000 & 0.0058 & 0.0056 & 0.0045 & 0.0043\\
    $50$ &100000 & 0.0112 & 0.0168 & 0.0151 & 0.0169\\
    $100$&100000 & 0.0152 & 0.0198 & 0.0173 & 0.0202\\
    \hline
  \end{tabular}
  \label{tab1}
\end{table}

\begin{figure}[htbp]
  \centering
    \includegraphics[width=0.45\textwidth]{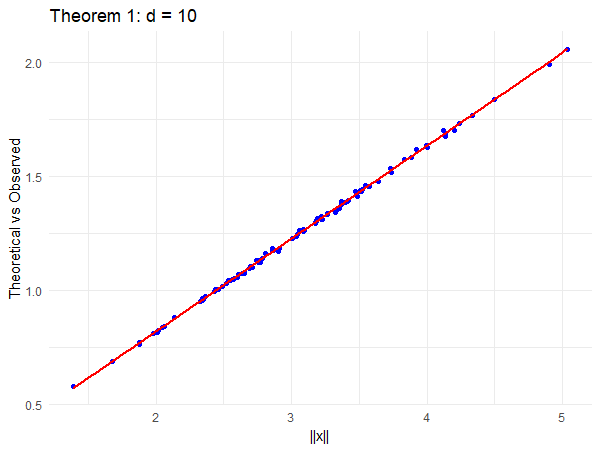}
    \includegraphics[width=0.45\textwidth]{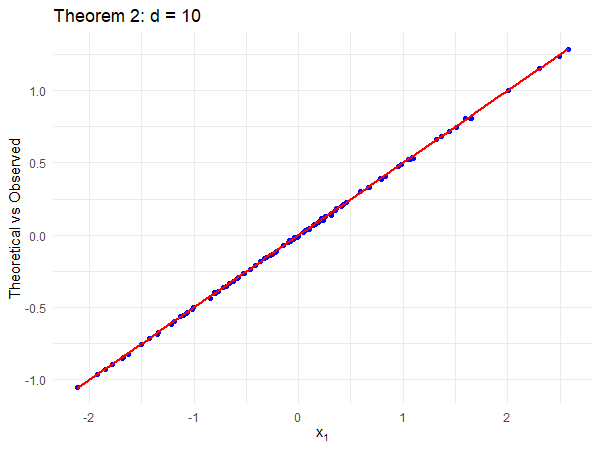}
    \vspace{1em}
    \includegraphics[width=0.45\textwidth]{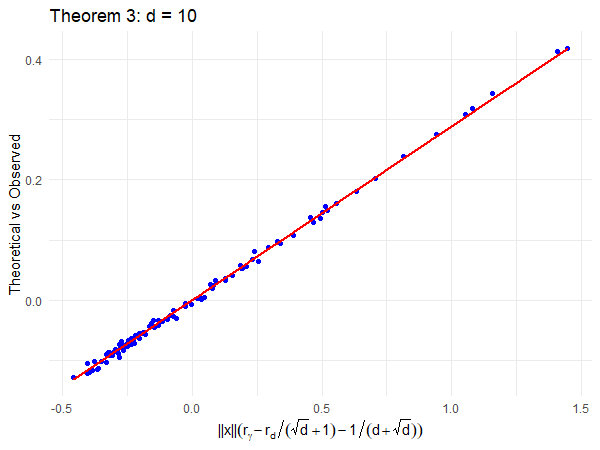}
    \includegraphics[width=0.45\textwidth]{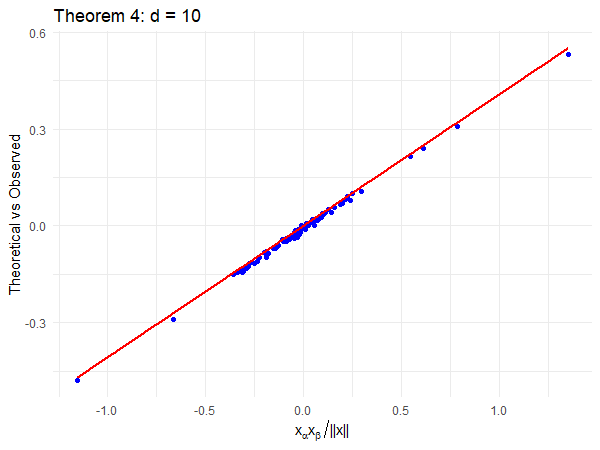}
  \caption{
  {Values of $X^\top v$ vs.\ theoretical $F(x)$ for Theorems \ref{thm1}–\ref{thm4} ($d=10,m=100000$).}
  %For approximate eigenvectors $v$ and their limiting functions $F$ from Theorems \ref{thm1} to \ref{thm4}, we show the values of $X^\top v$ against the theoretical values $F(x)$. The top left shows the case of $v^{(0)}$ (Group 1), top right shows the case of $v^{(l)}$ (Group 2), bottom left shows the case of $v^{(\gamma)}$ (Group 3), and the bottom right shows the case of $v^{(\alpha,\beta)}$ (Group 3). The plots shown are generated with $d = 10$ and $m = 100000$.
  }
  \label{fig1}
\end{figure}

\section{Conclusion and Discussion}\label{sec_conclusion}
This work advances the theoretical understanding of two-layer ReLU neural networks in the infinite-width limit by deriving the asymptotic functional forms of basis functions aligned with the approximate eigenvectors of the FIM. Our main results, encapsulated in Theorems \ref{thm1} to \ref{thm4}, reveal that these basis functions converge to forms proportional to $\Vert x \Vert$, $x_l$, $(x_\gamma^2-x_d^2/(\sqrt{d}+1))/\Vert x \Vert- \Vert x \Vert/(d+\sqrt{d})$, and $x_\alpha x_\beta/\Vert x\Vert$, prioritized by gradient descent due to their alignment with the FIM’s leading eigenvectors regardless of parametrization. The FIM induced inner product, which approximates orthogonality in the function space, establishes a novel connection between parameter and function spaces, offering insights into neural network expressivity and optimization dynamics. Simulations validate these theoretical approximations, with mean absolute errors decreasing as the number of hidden units ($m$) increases, confirming practical relevance.

These findings have significant implications for deep learning theory. By characterizing the functions learned early in training, our work provides a framework for analyzing how ReLU networks prioritize certain patterns, such as radial or coordinate-specific features, which can inform model design and initialization strategies. The connection to the Neural Tangent Kernel (NTK) suggests potential extensions to deeper architectures, where similar function space analyses could elucidate generalization properties.

%Currently, our work builds upon the approximate spectral eigendecomposition of the FIM in ReLU networks drawn from normal training data. This may limit its applicability to real-world datasets that may either be correlated or non-Gaussian, or if deeper networks are needed to encapsulate non-linear relationships. Also, the infinite-width limit ($m \to \infty$) simplifies analyses but may not fully capture the dynamics of finite-width networks, especially for small $m$ that are more prone to be affected by stochastic effects. We believe future works could explore generalizations including non-ReLU activations, deeper network structures, or general input distributions to broaden the applicability of these results.

\section{Limitations}

Our analysis has several limitations concerning the role of the input distribution $p_X$ and the modeling assumptions. 
First, while the convergence results in our theorems are independent of $p_X$—since they are determined by the random weight matrix $W$—the interpretation of the basis functions $v^{(0)}, v^{(l)}, \dots$ as approximate eigenvectors of the Fisher information matrix (FIM) crucially depends on $p_X$. 
If $p_X$ is rotationally invariant (e.g., the standard Gaussian), orthogonality relations such as $v^{(0)\top} J v^{(l)} = 0$ and norm equalities like $\langle F_0,F_0\rangle = (2d+1)/(4\pi)$ hold. 
For non-rotationally invariant $p_X$, these identities may no longer be valid, and the relative magnitudes of the functions $F_0, F_l, \dots$ can change, potentially altering the natural hierarchical structure suggested by our approximation \eqref{theapproximation}.

Second, our numerical evaluation of this effect is limited. 
When we approximated $p_X$ by empirical distributions sampled from the standard Gaussian ($d=10$ with sample sizes $1000$ and $100$), the eigenvalue distributions were close to the theoretical ones in the large-sample case (within $\pm 20\%$ error), but showed larger deviations for smaller samples (about $\pm 50\%$ error). 
Although no reversal of eigenvalue ordering was observed, this indicates that finite-sample effects or departures from Gaussianity may influence the hierarchy of the basis functions.

Third, our work currently builds upon the approximate spectral eigendecomposition of the FIM in two-layer ReLU networks under normally distributed training data. 
This may limit applicability to real-world datasets that are correlated or non-Gaussian, or to settings where deeper networks are needed to capture more complex non-linear relationships. 
Moreover, our analysis assumes the infinite-width limit ($m \to \infty$), which simplifies the theory but may not fully capture the dynamics of finite-width networks. 
In particular, smaller $m$ can be more sensitive to stochastic effects and deviate from the asymptotic predictions.

Finally, our current framework does not directly handle cases where the data are distributed on lower-dimensional manifolds within $\mathbb{R}^d$. 
In such settings, the FIM may become degenerate even under under-parameterization, and different analytical tools are required. 
We believe future work could explore generalizations beyond ReLU activations, deeper network architectures, finite-width regimes, and more general input distributions, which would broaden the applicability of our results.

\section*{Acknowledgments}
The authors give their sincere gratitude to Professors Hiroshi Nagaoka, Noboru Murata, and Kazushi Mimura for the valuable discussion with them. This work was supported by JSPS KAKENHI, Grant Number JP23H05492, JP25K24611. 

{
\small
%\bibliography{bib}

}

\newpage
\section*{Appendix}
\appendix
\section{Proof of Theorem \ref{thm1}}\label{AppA}
Let $v^{(0)} \coloneqq \left(\Vert W^{(1)}\Vert/\sqrt{d},...,\Vert W^{(m)}\Vert/\sqrt{d}\right)$, with $W^{(i)} \in \R^d$ for $i = 1,...,m$. First note that for any positive constant $a > 0$, $\sigma(ax) = a \sigma(x)$. By rewriting and applying the weak law of large numbers, we get
\begin{align*}
    X^\top v^{(0)} &= \sum_{i=1}^m \sigma (x^\top W^{(i)}) \frac{\Vert W^{(i)} \Vert}{\sqrt{d}}\\
    &= \sum_{i=1}^m \sigma \left(\frac{x^\top Z^{(i)}}{\sqrt{m}}\right) \frac{\Vert Z^{(i)} \Vert}{\sqrt{dm}}\\
    &= \frac{1}{m} \sum_{i=1}^m \sigma (x^\top Z^{(i)}) \frac{\Vert Z^{(i)} \Vert}{\sqrt{d}}\\
    &\cip \E\left[\sigma(x^\top Z) \frac{\Vert Z \Vert}{\sqrt{d}}\right],
\end{align*}
where $Z^{(i)} \sim N(0,I_d)$ are independent for $i = 1,...,m$, and the expectation is taken with respect to $Z \sim N(0,I_d)$. Evaluating the expectation explicitly, we have \[ \E\left[\sigma(x^\top Z) \frac{\Vert Z \Vert}{\sqrt{d}}\right]
  = \E\left[\sigma\left(x^\top \frac{Z}{\Vert Z \Vert}\right) \frac{\Vert Z \Vert^2}{\sqrt{d}}\right] = \E\left[\sigma\left(x^\top \frac{Z}{\Vert Z \Vert}\right) \right] \E \left[\frac{\Vert Z \Vert^2}{\sqrt{d}}\right]\]
by the independence of the magnitude $\Vert Z \Vert^2$ and direction $Z/\Vert Z \Vert$ of $Z$. The spherical symmetry of the distribution of $ {Z}/\Vert Z \Vert$ means we may assume $x = \Vert x \Vert(1,0,0,...,0)$ to evaluate the first expectation. Denoting $\hat Z \coloneqq Z/\Vert Z \Vert$, the marginal density of $\hat Z_1$ is given by \[f_{\hat Z_1}(u) = \frac{(1 - u^2)^{\frac{d-1}{2}-1}}{B(\frac{d-1}{2},\frac{1}{2})},\] where $B(\cdot,\cdot)$ is the beta function. Hence, 
\begin{align*}
    \E\left[\sigma(x^\top \hat Z)\right] =  \E\left[\sigma\left(\Vert x \Vert \hat Z_1\right)\right] &= \Vert x \Vert \E\left[\sigma(\hat Z_1)\right] \\
    &= \Vert x \Vert \int_{-1}^1 \sigma(u) \frac{(1 - u^2)^{\frac{d-1}{2}-1}}{B(\frac{d-1}{2},\frac{1}{2})} du\\
    &= \Vert x \Vert \int_{0}^1 u \frac{(1 - u^2)^{\frac{d-1}{2}-1}}{B(\frac{d-1}{2},\frac{1}{2})} du\\
    &= \Vert x \Vert \left[\frac{-(1-u^2)^{\frac{d-1}{2}}}{(d-1)B(\frac{d-1}{2},\frac{1}{2})}\right]^1_0\\
    &= \frac{\Vert x \Vert}{(d-1)B(\frac{d-1}{2},\frac{1}{2})}\\
    &= \frac{1}{2\pi}B\Bigl(\frac{d}{2},\frac{1}{2}\Bigr)\Vert x \Vert,
\end{align*}
where $\sigma\left(\Vert x \Vert \hat Z_1\right) = \Vert x \Vert \sigma(\hat Z_1)$ follows from the non-negativity of $\Vert x \Vert$. The second term  \[\E\left[ \frac{\Vert Z \Vert^2}{\sqrt{d}}\right] = \sqrt{d}\] is straightforward since $\Vert Z \Vert^2$ is $\chi$-squared distributed with $d$ degrees of freedom.
Combining gives 
\begin{align*}
    X^\top v^{(0)} \cip \frac{\sqrt{d}}{2\pi}B\Bigl(\frac{d}{2},\frac{1}{2}\Bigr)\Vert x \Vert.
\end{align*}
\section{Proof of Theorem \ref{thm2}}\label{AppB}
The second group consists of $d$ eigenvectors of the form $v^{(l)} \coloneqq W_l$ for $l = 1,...,d$, where $W_l$ is the $l$'s row of the weight matrix $W$. Following the same argument in Appendix \ref{AppA} and by noting that $W_{li} = W^{(i)}_l$ and $Z_{li} = Z^{(i)}_l$, we get
\begin{align*}
    X^\top v^{(l)} &= \sum_{i=1}^m \sigma (x^\top W^{(i)})W^{(i)}_{l}\\
    &= \sum_{i=1}^m \sigma \left(\frac{x^\top Z^{(i)}}{\sqrt{m}}\right) \frac{Z^{(i)}_{l}}{\sqrt{m}}\\
    &= \frac{1}{m} \sum_{i=1}^m \sigma (x^\top Z^{(i)}) Z^{(i)}_{l}\\
    &\cip \E\left[\sigma(x^\top Z) Z_{l}\right],
\end{align*}
where the expectation is taken with respect to $Z \sim N(0,I_d)$. Write the last expectation as \[
\E\left[\sigma(x^\top Z) Z_{l}\right] = \E\left[ \sigma\left(x_lZ_l + \sum_{k \neq l} x_kZ_k\right)Z_l\right].
\]
Here,  note that $\sum_{k \neq l} x_kZ_k$ is normally distributed and has the same distribution as $\Vert x_{-l} \Vert \tilde Z$, where $\tilde Z \sim N(0,1)$ is independent of $Z_l$. Appealing to the tower law of expectation then yields
\begin{equation}\label{G2Eq1}
    \E\left[ \sigma\left(x_lZ_l + \sum_{k \neq l} x_kZ_k\right)Z_l\right] =  \E_{Z_l} \left[\E_{\tilde Z}\left\{Z_l \sigma(x_lZ_l +  \Vert x_{-l} \Vert \tilde Z )\vert Z_l\right\}\right].
\end{equation}
Here, if $\Vert x_{-l} \Vert = 0$ and if $x_l \geq 0$, the calculation simplifies to 
\begin{align*}
    \E\left[Z_l\sigma\left(x_lZ_l\right)\right] &= \int_0^\infty x_lu^2\phi(u) du\\
    &= x_l \left[u\phi(u)\right]_0^\infty + x_l\int_0^\infty \phi(u) du\\
    &= \frac{x_l}{2}.
\end{align*}
Note that if $x_l < 0$, \[\E\left[Z_l\sigma\left\{x_lZ_l\right\}\right] = \E\left[(-Z_l)\sigma\left\{x_l(-Z_l)\right\}\right] = -\E\left[Z_l\sigma\left((-x_l)-Z_l\right)\right] = -\frac{-x_l}{2} = \frac{x_l}{2}.\]
Therefore, we may assume that $x_l > 0$ without loss of generality.
In general, when $\Vert x_{-l} \Vert \neq 0$, the following lemma is useful:
\begin{lem}\label{lemB1}
    Let $Z \sim N(0,1)$, $a,b \in \R$ be fixed constants with $b \neq 0$, and denote $\phi$ and $\Phi$ as the density and cumulative distribution functions of $Z$ respectively. Then \[\E[\sigma(a+bZ)] = a\Phi\left(\frac{a}{|b|}\right) + |b|\phi\left(\frac{a}{|b|}\right).\]
\end{lem}
\begin{proof}
First note that $\E[\sigma(a+bZ)] = \E[\sigma(a - bZ)] = \E[\sigma(a+|b|Z)]$ because the law of $Z$ and $-Z$ are the same. Then, assuming $b > 0$, we have
    \begin{align*}
        \E[\sigma(a+bZ)] &= \int_{-a/b}^\infty (a+bz) \phi(z) dz \\
        &= \int_{-a/b}^\infty a \phi(z) dz + b\int_{-a/b}^\infty z \phi(z)dz\\
        &= a \left(1 -\Phi \left(\frac{-a}{b}\right)\right) + b\int_{-a/b}^\infty z \cdot \frac{1}{\sqrt{2\pi}} e^{-z^2/2} dz \\
        &= a\Phi \left(\frac{a}{b}\right) - b\left[\phi(z)\right]^\infty_{-a/b}\\
        &= a\Phi\left(\frac{a}{b}\right) + b\phi\left(\frac{a}{b}\right)
    \end{align*}
Therefore, for any $b \neq 0$, \[\E[\sigma(a+bZ)] =  a\Phi\left(\frac{a}{|b|}\right) + |b|\phi\left(\frac{a}{|b|}\right).\]
\end{proof}
Applying Lemma \ref{lemB1} with $a = x_lZ_l$ and $b = \Vert x_{-l}\Vert$ simplifies the RHS of \eqref{G2Eq1} to 
\begin{align*}
    &\E_{Z_l} \left[Z_l \left(x_lZ_l\Phi\left(\frac{x_lZ_l}{\Vert x_{-l}\Vert}\right) + \Vert x_{-l}\Vert\phi\left(\frac{x_lZ_l}{\Vert x_{-l}\Vert}\right)\right)\right]\\
    =&\E\left[x_l Z_l^2 \Phi\left(\frac{x_lZ_l}{\Vert x_{-l}\Vert}\right)\right] +\Vert x_{-l}\Vert\E\left[Z_l\phi\left(\frac{x_lZ_l}{\Vert x_{-l}\Vert}\right)\right] \\
    =&\E\left[x_l Z_l^2 \Phi\left(\frac{x_lZ_l}{\Vert x_{-l}\Vert}\right)\right],
\end{align*}
where $\E\left[Z_l\phi\left(x_lZ_l/\Vert x_{-l}\Vert\right)\right] = 0$ because $u \phi(C u)\phi(u)$ is an integrable and odd function on the real line for any $C \in \R$. Evaluating the expectation and switching the order of integration then give,
\begin{align*}
    \E\left[x_l Z_l^2 \Phi\left(\frac{x_lZ_l}{\Vert x_{-l}\Vert}\right)\right] &= x_l \int_{-\infty}^\infty z_l^2 \phi(z_l) \left(\int_{-\infty}^{x_lz_l/\Vert x_{-l}\Vert} \phi(u) du \right)dz_l\\
    &= x_l \int_{u = -\infty}^{u = \infty} \phi(u) \left(\int_{z_l = u\Vert x_{-l}\Vert/x_l}^\infty z_l^2 \phi(z_l) dz_l\right) du\\
    &= x_l \int_{-\infty}^{\infty} \phi(u) \left[-z_l\phi(z_l) + \Phi(z_l)\right]^\infty_{u\Vert x_{-l}\Vert/x_l} du\\
    &= x_l \int_{-\infty}^{\infty} \phi(u)\left(1 + \frac{u\Vert x_{-l}\Vert}{x_l}\phi\left(\frac{u\Vert x_{-l}\Vert}{x_l}\right) - \Phi\left(\frac{u\Vert x_{-l}\Vert}{x_l}\right)\right) du\\
    &= x_l - x_l\int_{-\infty}^{\infty} \phi(u)\Phi\left(\frac{u\Vert x_{-l}\Vert}{x_l}\right) du,
\end{align*}
where the second term in the parenthesis in the penultimate line can be deleted because it is an odd and integrable function. Finally, the same argument also applies to integrating $\phi(u)\left(\Phi\left(u\Vert x_{-l}\Vert/x_l\right) - 1/2\right)$, thus
\begin{align*}
    \int_{-\infty}^{\infty} \phi(u)\Phi\left(\frac{u\Vert x_{-l}\Vert}{x_l}\right) du &= \int_{-\infty}^{\infty} \phi(u)\left(\Phi\left(\frac{u\Vert x_{-l}\Vert}{x_l}\right) - \frac{1}{2}\right) + \frac{\phi(u)}{2} du\\
    &= \int_{-\infty}^{\infty} \frac{\phi(u)}{2} du\\
    &= \frac{1}{2}.
\end{align*}
Combining the results yields \[X^\top v^{(l)} \cip \frac{x_l}{2},\]
which completes the proof of Theorem \ref{thm2}.

\section{Proof of Lemma \ref{thm3}}\label{AppC}
We first evaluate the limit:
\begin{align*}
   X^\top v^{(\gamma,\gamma)} &= \sum_{i=1}^m \sigma (x^\top W^{(i)})\frac{\sqrt{d+2}W_{\gamma}^{(i)2}}{\Vert W^{(i)} \Vert}\\ 
   &=\sqrt{d+2} \sum_{i=1}^m \sigma \left(\frac{x^\top Z^{(i)}}{\sqrt{m}}\right) \frac{Z^{(i)2}_{\gamma}}{\sqrt{m}\Vert Z^{(i)}\Vert}\\ 
   &= \frac{\sqrt{d+2}}{m} \sum_{i=1}^m \sigma (x^\top Z^{(i)}) \frac{Z^{(i)2}_{\gamma}}{\Vert Z^{(i)}\Vert}\\
   &\cip \sqrt{d+2}\E\left[\sigma(x^\top Z) \frac{Z^2_{\gamma}}{\Vert Z \Vert}\right].
\end{align*}
The above expectation is calculated as
\[
\E\left[\sigma(x^\top Z) \frac{Z^2_{\gamma}}{\Vert Z \Vert}\right]
=\frac{d}{2\pi(d+1)}
   B\!\left(\frac{d}{2},\frac{1}{2}\right)\left(\frac{x_\gamma^2}{\|x\|}\,+\|x\|\right),
\]
which will be shown in Section \ref{comp_exp}.
Then, we obtain
\begin{align*}
     X^\top v^{(\gamma,\gamma)} \cip \frac{d\sqrt{d+2}}{2\pi(d+1)}
   B\!\left(\frac{d}{2},\frac{1}{2}\right)\left(\frac{x_\gamma^2}{\|x\|}\,+\|x\|\right).
\end{align*}
By (4)  and  (6) in the main paper, 
\begin{align*}
    X^\top \tilde{v}^{(\gamma)}
    &\cip \frac{1}{\sqrt{2}}\left( \frac{d\sqrt{d+2}}{2\pi(d+1)}
   B\!\left(\frac{d}{2},\frac{1}{2}\right)\left(\frac{x_\gamma^2}{\|x\|}\,+\|x\|\right) - \sqrt{\frac{d+2}{d}}\frac{\sqrt{d}}{2\pi}B\!\left(\frac{d}{2},\frac{1}{2}\right)\|x\| \right)\\
    &= \frac{1}{\sqrt{2}} \cdot \frac{d\sqrt{d+2}}{2\pi(d+1)}
    B\!\left(\frac{d}{2},\frac{1}{2}\right)
    \left(
    \frac{x_\gamma^2}{\|x\|}+\|x\|-\frac{d+1}{d}\|x\|
    \right)\\
    &=\frac{d\sqrt{d+2}}{2\pi(d+1)\sqrt{2}}
    B\!\left(\frac{d}{2},\frac{1}{2}\right)
    \left( \frac{x_\gamma^2}{\|x\|}- \frac{\|x\|}{d}\right).
\end{align*}

\section{Proof of Theorem \ref{thm3mod}}
By (3) in the main paper and Lemma \ref{thm3}, we have
\begin{align*}
    X^\top v^{(\gamma)} &= X^\top \tilde{v}^{(\gamma)}
    - \frac{1}{\sqrt{d}+1} X^\top \tilde{v}^{(d)} \\
    &\cip \frac{d\sqrt{d+2}}{2\pi(d+1)\sqrt{2}} 
    B\left(\frac{d}{2}, \frac{1}{2}\right) \Biggl( 
        \frac{x_\gamma^2}{\|x\|} - \frac{1}{\sqrt{d}+1} \frac{x_d^2}{\|x\|} 
        - \frac{1}{d} \left( 1 - \frac{1}{\sqrt{d}+1} \right) 
    \Biggr) \\
    &= \frac{d\sqrt{d+2}}{2\pi(d+1)\sqrt{2}} 
    B\left(\frac{d}{2}, \frac{1}{2}\right) \Biggl( 
        \frac{x_\gamma^2}{\|x\|} - \frac{1}{\sqrt{d}+1} \frac{x_d^2}{\|x\|} 
        - \frac{1}{d+\sqrt{d}} 
    \Biggr).
\end{align*}

\section{Proof of Theorem \ref{thm4}}\label{AppD}
\noindent We evaluate the limit 
\begin{align*}
    X^\top v^{(\alpha,\beta)} \cip \sqrt{d+2}\E\left[\sigma(x^\top Z) \frac{Z_\alpha Z_\beta}{\Vert Z \Vert}\right].
\end{align*}
The above expectation is calculated as
\[
\E\left[\sigma(x^\top Z) \frac{Z_{\alpha}Z_{\beta}}{\Vert Z \Vert}\right]
=\frac{d}{2\pi(d+1)}
   B\!\left(\frac{d}{2},\frac{1}{2}\right)\frac{x_\alpha x_\beta}{\|x\|},
\]
which will be shown in Section \ref{comp_exp}.
Then, we obtain
\begin{align*}
     X^\top v^{(\alpha,\beta)} \cip \frac{d\sqrt{d+2}  }{2(d+1)\pi}B\Bigl(\frac{d}{2},\frac{1}{2}\Bigr)\frac{x_\alpha x_\beta}{\Vert x\Vert} .
\end{align*}

\section{Computation of some expectations}\label{comp_exp}
In this section, we compute the expectations of 
\[
\E\left[\sigma(x^\top Z) \frac{Z_{\alpha}Z_{\beta}}{\Vert Z \Vert}\right] \qquad\text{and}\qquad \E\left[\sigma(x^\top Z) \frac{Z^2_{\gamma}}{\Vert Z \Vert}\right].
\]
Without loss of generality, we set $\alpha=1$, $\beta=2$, and $\gamma=1$.

First, let us compute 
\begin{align}
\mathbb{E}_{Z}\!\left[ \sigma(x^\top Z) \frac{Z_{1} Z_{2}}{\|Z\|} \right],
\label{eq:Zexp}
\end{align}
where $Z = (Z_{1}, \ldots, Z_{d})^\top$, and $\{e_i\}_{i=1}^d$ denotes the standard basis of $\mathbb{R}^d$. Namely, $Z_i=e_i^T Z$ for $i=1,\ldots, d$.

We define an orthogonal matrix $U = (u_1, \ldots, u_d)$ such that 
\[
Ux = \|x\|e_1.
\]
Since the distribution of $Z$ is rotationally invariant, \eqref{eq:Zexp} does not change if we replace $Z$ with $U^T Z$ in the expectation.  
Hence,
\begin{align}
\mathbb{E}_{Z}\!\left[ \sigma(x^\top Z) \frac{Z_{1} Z_{2}}{\|Z\|} \right]
&= \mathbb{E}_{Z}\!\left[ \sigma(x^\top U^\top Z) 
      \frac{(U^\top Z)_1 (U^\top Z)_2}{\|Z\|} \right] \\
&= \mathbb{E}_{Z}\!\left[ \sigma(\|x\|e_1^\top Z) 
      \frac{(u_1^\top Z) (u_2^\top Z)}{\|Z\|} \right] \\
&= \|x\| \, \mathbb{E}_{\hat Z,\|Z\|}\!\left[ 
      \sigma( \hat Z_1)\, (u_1^\top \hat Z) (u_2^\top \hat Z)  \|Z\|^2 
   \right],
\end{align}
where we defined $\hat Z = Z/ \|Z\|$.  
Since $\hat Z$ and $\|Z\|$ are independent,
\begin{align}
\mathbb{E}_{Z}\!\left[ \sigma(x^\top Z) \frac{Z_{1} Z_{2}}{\|Z\|} \right]
&= \|x\|\, \mathbb{E}_{\hat Z}\!\left[
    \sigma(\hat Z_1) (u_1^\top \hat Z)(u_2^\top \hat Z)
  \right]
  \mathbb{E}_{\|Z\|}\!\left[\|Z\|^2\right].
\end{align}

Because $\|Z\|^2 \sim \chi^2_d$, we have $\mathbb{E}_{\|Z\|}\!\left[\|Z\|^2\right] = d$.
Therefore,
\begin{align}
\mathbb{E}_{Z}\!\left[ \sigma(x^\top Z) \frac{Z_{1} Z_{2}}{\|Z\|} \right]
= d\,\|x\|\, \mathbb{E}_{\hat Z}\!\left[
    \sigma(\hat Z_1) (u_1^\top \hat Z)(u_2^\top \hat Z)
  \right].
\label{eq:main_reduced}
\end{align}

Next, express $u_1$ and $u_2$ in the standard basis:
\[
u_1 = \sum_{i=1}^d c_i e_i, 
\qquad
u_2 = \sum_{j=1}^d d_j e_j.
\]
Then,
\begin{align}
\mathbb{E}_{\hat Z}\!\left[
  \sigma(\hat Z_1) (u_1^\top \hat Z)(u_2^\top \hat Z)
\right]
&= \mathbb{E}_{\hat Z}\!\left[
    \sigma(\hat Z_1)
    \left(\sum_i c_i \hat Z_i\right)
    \left(\sum_j d_j \hat Z_j\right)
  \right] \\
&= \sum_{i,j} c_i d_j \,
   \mathbb{E}_{\hat Z}\!\left[
     \sigma(\hat Z_1)\hat Z_i\hat Z_j
   \right].
\end{align}

Due to the symmetry of $\hat Z$ on the unit sphere, the cross terms ($i \neq j$) vanish:
\[
\mathbb{E}_{\hat Z}\!\left[
  \sigma(\hat Z_1)\hat Z_i\hat Z_j
\right] = 0, \quad (i \ne j).
\]
Hence only the diagonal terms remain:
\[
\mathbb{E}_{\hat Z}\!\left[
  \sigma(\hat Z_1) (u_1^\top \hat Z)(u_2^\top \hat Z)
\right]
= \sum_i c_i d_i \,
  \mathbb{E}_{\hat Z}\!\left[
    \sigma(\hat Z_1)\hat Z_i^2
  \right].
\]

Let
\[
C_d = \mathbb{E}_{\hat Z}\!\left[
  \sigma(\hat Z_1)\hat Z_1^2
\right], 
\qquad 
C_d' = \mathbb{E}_{\hat Z}\!\left[
  \sigma(\hat Z_1)\hat Z_2^2
\right].
\]
Since $\sum_i c_i d_i = 0$ (because $u_1$ and $u_2$ are orthogonal), we have
\[
c_1 d_1 + \sum_{i\ge 2} c_i d_i = 0 
\quad\Rightarrow\quad 
\sum_{i\ge 2} c_i d_i = -c_1 d_1.
\]
Since $c_i=u_1^T e_i$ and $x=\|x\|U^T e_1$, we have $c_1=x_1/\|x\|$. Similarly, we have $d_j=u_2^T e_j$ and $d_1=x_2/\|x\|$, which gives 
$c_1 d_1 = x_1 x_2 / \|x\|^2$.

Thus,
\begin{align}
\mathbb{E}_{\hat Z}\!\left[
  \sigma(\hat Z_1)(u_1^\top \hat Z)(u_2^\top \hat Z)
\right]
&= c_1 d_1 C_d + \sum_{i\ge 2} c_i d_i C_d' 
 = c_1 d_1 (C_d - C_d') \\
&= \frac{x_1 x_2}{\|x\|^2} (C_d - C_d').
\end{align}

Finally, substituting this into \eqref{eq:main_reduced}, we obtain
\[
\mathbb{E}_{Z}\!\left[
  \sigma(x^\top Z)\frac{Z_1 Z_2}{\|Z\|}
\right]
= d(C_d - C_d')\,\frac{x_1 x_2}{\|x\|}.
\]
Using the result of the computation of $C_d$ and $C_d'$ in the next section, we obtain
\[
\mathbb{E}_{Z}\!\left[
  \sigma(x^\top Z)\frac{Z_1^2}{\|Z\|}
\right]
= \frac{d}{2\pi(d+1)}
   B\!\left(\frac{d}{2},\frac{1}{2}\right)\frac{x_1 x_2}{\|x\|}.
\]

Next, let us compute 
\begin{align}
\mathbb{E}_{Z}\!\left[ \sigma(x^\top Z) \frac{Z_{1}^2}{\|Z\|} \right].
\label{eq:Zexp2}
\end{align}
Similarly, we have
\begin{align}
\mathbb{E}_{Z}\!\left[ \sigma(x^\top Z) \frac{Z_{1}^2}{\|Z\|} \right]
&= \mathbb{E}_{Z}\!\left[ \sigma(x^\top U^\top Z) 
      \frac{(U^\top Z)_1^2}{\|Z\|} \right] \\
&= \mathbb{E}_{Z}\!\left[ \sigma(\|x\|e_1^\top Z) 
      \frac{(u_1^\top Z)^2}{\|Z\|} \right] \\
&= \|x\| \, \mathbb{E}_{\hat Z,\|Z\|}\!\left[ 
      \sigma( \hat Z_1)\, (u_1^\top \hat Z)^2  \|Z\|^2 
   \right]\\
&= d\,\|x\|\, \mathbb{E}_{\hat Z}\!\left[
    \sigma(\hat Z_1) (u_1^\top \hat Z)^2
  \right].
  \label{eq:main_reduced2}
\end{align}
Then,
\begin{align}
\mathbb{E}_{\hat Z}\!\left[
  \sigma(\hat Z_1) (u_1^\top \hat Z)^2
\right]
&= \mathbb{E}_{\hat Z}\!\left[
    \sigma(\hat Z_1)
    \left(\sum_i c_i \hat Z_i\right)^2
  \right] \\
&= \sum_{i,j} c_i c_j \,
   \mathbb{E}_{\hat Z}\!\left[
     \sigma(\hat Z_1)\hat Z_i\hat Z_j
   \right]\\  
&= \sum_i c_i^2 \,
  \mathbb{E}_{\hat Z}\!\left[
    \sigma(\hat Z_1)\hat Z_i^2
  \right].\\
  &= c_1^2 \,C_d
  +\sum_{i=2}^d c_i^2 \,
  C'_d.
\end{align}
Since
\[
u_1 = \sum_{i=1}^d c_i e_i,
\]
we have
\[
\sum_{i=1}^d c_i^2=1.
\]
Together with $c_1=x_1/\|x\|$, we have
\begin{align}
\mathbb{E}_{\hat Z}\!\left[
  \sigma(\hat Z_1) (u_1^\top \hat Z)^2
\right]
&= \frac{x_1^2}{\|x\|^2} C_d
  +\left(1-\frac{x_1^2}{\|x\|^2}\right)
  C'_d\\
&= \frac{x_1^2}{\|x\|^2}(C_d-C'_d)
  +C'_d.
\end{align}
Finally, substituting this into \eqref{eq:main_reduced2}, we obtain
\[
\mathbb{E}_{Z}\!\left[
  \sigma(x^\top Z)\frac{Z_1^2}{\|Z\|}
\right]
= d\left((C_d - C_d')\frac{x_1^2}{\|x\|}\,+C'_d\, \|x\|\right).
\]

Using the result of the computation of $C_d$ and $C_d'$ in the next section, we obtain
\[
\mathbb{E}_{Z}\!\left[
  \sigma(x^\top Z)\frac{Z_1^2}{\|Z\|}
\right]
= \frac{d}{2\pi(d+1)}
   B\!\left(\frac{d}{2},\frac{1}{2}\right)\left(\frac{x_1^2}{\|x\|}\,+\|x\|\right).
\]

\section{Computation of $C_d$ and $C_d'$}

In this section, we compute the quantities
\[
C_d \;=\; \mathbb{E}_{\hat Z}\!\left[ \sigma(\hat Z_1)\, \hat Z_1^2 \right],
\qquad
C_d' \;=\; \mathbb{E}_{\hat Z}\!\left[ \sigma(\hat Z_1)\, \hat Z_2^2 \right],
\]
where $\hat Z = (\hat Z_1,\ldots,\hat Z_d)$ is uniformly distributed on the unit sphere $S^{d-1}$ in $\mathbb{R}^d$, and $\sigma(t)=\max\{t,0\}$ is the ReLU function.
Due to the ReLU, only the region $\hat Z_1 > 0$ contributes.

Applying Theorem 2.1 in \cite{SongGupta1997} with $n=d$, $k=2$, and $p=2$, we obtain that the marginal density of $(\hat Z_1,\hat Z_2)$ is
\begin{equation}
p(z_1,z_2)
= \kappa_d \left( 1 - z_1^2 - z_2^2 \right)^{\frac{d-4}{2}},
\qquad
z_1^2 + z_2^2 \le 1,
\label{eq:joint-marginal}
\end{equation}
with the normalization constant
\begin{equation}
\kappa_d
= \frac{\Gamma\!\left(\frac{d}{2}\right)}{\pi\,\Gamma\!\left(\frac{d-2}{2}\right)}.
\label{eq:cd-const}
\end{equation}

We first compute
\[
C_d
= \mathbb{E}\!\left[ \sigma(Z_1)\, Z_1^2 \right]
= \int_{z_1 > 0} z_1 \cdot z_1^2 \cdot p(z_1,z_2)\, dz_1 dz_2
= \int_{z_1 > 0} z_1^3\, p(z_1,z_2)\, dz_1 dz_2.
\]

We change to polar coordinates in the $(z_1,z_2)$-plane:
\[
z_1 = r\cos\theta, \qquad
z_2 = r\sin\theta, \qquad
0 \le r \le 1,
\]
and the Jacobian is $r\,dr\,d\theta$.
The condition $z_1>0$ corresponds to $\cos\theta > 0$, i.e.,
\[
-\frac{\pi}{2} \le \theta \le \frac{\pi}{2}.
\]
Moreover,
\[
z_1^3 = (r\cos\theta)^3 = r^3 \cos^3\theta,
\qquad
1 - z_1^2 - z_2^2 = 1 - r^2.
\]
Therefore,
\begin{align}
C_d
&= \int_{-\pi/2}^{\pi/2} \int_0^1
   \bigl(r^3 \cos^3\theta\bigr)\,
   \kappa_d \bigl(1-r^2\bigr)^{\frac{d-4}{2}}\,
   r\,dr\,d\theta \\
&= \kappa_d
   \left(
   \int_{-\pi/2}^{\pi/2} \cos^3\theta\, d\theta
   \right)
   \left(
   \int_0^1 r^4 \bigl(1-r^2\bigr)^{\frac{d-4}{2}} dr
   \right).
\label{eq:Cd-separate}
\end{align}
A standard computation gives
\begin{equation}
\int_{-\pi/2}^{\pi/2} \cos^3\theta\, d\theta
= \frac{4}{3}.
\label{eq:ang-int-Cd}
\end{equation}
Define
\[
R_d
= \int_0^1 r^4 \bigl(1-r^2\bigr)^{\frac{d-4}{2}} dr.
\]
We use the substitution $u = r^2$, then
\[
r^4 dr =  \frac{u^{3/2}}{2}du.
\]
Hence
\begin{align}
R_d
&= \frac{1}{2} \int_0^1
    u^{3/2} (1-u)^{\frac{d-4}{2}} \, du \\
&= \frac{1}{2}
   B\!\left(\frac{5}{2}, \frac{d-2}{2}\right)
 = \frac{1}{2}
   \frac{\Gamma\!\left(\frac{5}{2}\right)
         \Gamma\!\left(\frac{d-2}{2}\right)}%
        {\Gamma\!\left(\frac{d+3}{2}\right)}.
\end{align}
Using $\Gamma(\tfrac{5}{2}) = \tfrac{3\sqrt{\pi}}{4}$, we obtain
\begin{equation}
R_d
= \frac{3\sqrt{\pi}}{8}
  \frac{\Gamma\!\left(\frac{d-2}{2}\right)}%
       {\Gamma\!\left(\frac{d+3}{2}\right)}.
\label{eq:radial-int}
\end{equation}

Substituting \eqref{eq:ang-int-Cd} and \eqref{eq:radial-int}
into \eqref{eq:Cd-separate}, we obtain
\begin{align}
C_d
&= \kappa_d \cdot \frac{4}{3} \cdot
    \frac{3\sqrt{\pi}}{8}
    \frac{\Gamma(\frac{d-2}{2})}{\Gamma(\frac{d+3}{2})} \\
&= \kappa_d \cdot
    \frac{\sqrt{\pi}}{2}
    \frac{\Gamma(\frac{d-2}{2})}{\Gamma(\frac{d+3}{2})}.
\end{align}
Finally, substituting \eqref{eq:cd-const},
\[
\kappa_d
= \frac{\Gamma(\frac{d}{2})}{\pi\,\Gamma(\frac{d-2}{2})},
\]
we obtain
\begin{align}
C_d
&= \frac{1}{\pi}\frac{\sqrt{\pi}}{2}
   \frac{\Gamma\!\left(\frac{d}{2}\right)}%
        {\Gamma\!\left(\frac{d+3}{2}\right)}\\
        &= \frac{1}{\pi}
   B\!\left(\frac{d}{2},\frac{3}{2}\right)\\
&= \frac{1}{\pi(d+1)}
   B\!\left(\frac{d}{2},\frac{1}{2}\right).
\end{align}

Next, we compute
\[
C_d'
= \mathbb{E}\!\left[ \sigma(Z_1)\, Z_2^2 \right]
= \int_{z_1>0} z_1 \cdot z_2^2 \cdot p(z_1,z_2)\, dz_1 dz_2.
\]
In polar coordinates $z_1 = r\cos\theta$, $z_2 = r\sin\theta$ as above, we have
\[
z_1 \cdot z_2^2
= (r\cos\theta)\,(r\sin\theta)^2
= r^3 \cos\theta \sin^2\theta.
\]
Including the Jacobian $r\,dr\,d\theta$ and the density factor
$(1-r^2)^{\frac{d-4}{2}}$, we obtain
\begin{align}
C_d'
&= \int_{-\pi/2}^{\pi/2} \int_0^1
    \left( r^3 \cos\theta \sin^2\theta \right)
    \kappa_d \left(1-r^2\right)^{\frac{d-4}{2}}
    r\,dr\,d\theta \\
&= \kappa_d
   \left(
   \int_{-\pi/2}^{\pi/2} 
      \cos\theta \sin^2\theta\, d\theta
   \right)
   \left(
   \int_0^1 r^4 (1-r^2)^{\frac{d-4}{2}} dr
   \right).\\
   &= \kappa_d
   \, \cdot\frac{2}{3}\,
   R_d.
\end{align}
Thus, we obtain
\begin{align}
C'_d
&= \frac{C_d}{2}= \frac{1}{2\pi(d+1)}
   B\!\left(\frac{d}{2},\frac{1}{2}\right).
\end{align}

\vfill

\end{document}